%% file: paper.tex
\documentclass[10pt,twocolumn]{article} 
\usepackage{simpleConference}

\usepackage{amsmath}
\usepackage{wasysym}
\usepackage{amssymb}
\usepackage{amssymb}
\usepackage{mathtools}
\usepackage{wrapfig}
\usepackage{csvsimple}
\usepackage{caption}
\usepackage{subfig}
\usepackage{makecell}
\usepackage{hyperref}
\usepackage{natbib}

\usepackage[ruled, vlined, nofillcomment, linesnumbered, algo2e]{algorithm2e}
\usepackage{pgfplotstable, pgfmath} 
\usepackage{pgfplots} 
\usepackage{enumitem}

\usepackage{thmtools}

\usepackage[noenumitem,notheorems]{eda} 
\usepackage{prettyplots} 
\SetKw{Continue}{continue}
\SetKw{Break}{break}
\graphicspath{ {./figs/} }
\pgfplotsset{compat=1.18} 
\usetikzlibrary{calc} 
\usetikzlibrary{backgrounds}

\ifpdf
\usetikzlibrary{pgfplots.statistics, pgfplots.colorbrewer, patterns, external, arrows, calc,pgfplots.groupplots}

\include{defines}

\begin{document}

\title{Succinct Interaction-Aware Explanations}

\author{Sascha Xu, Joscha Cüppers, Jilles Vreeken \\
\\
\emph{CISPA Helmholtz Center for Information Security} \\
Saarbrücken, Germany \\
\{sascha.xu, joscha.cueppers, vreeken\}@cispa.de  \\
}

\maketitle
\hypersetup{
  pdfinfo={Title={Succinct Interaction-Aware Explanations},Author={Sascha Xu, Joscha Cüppers, Jilles Vreeken},Keywords={Explainability, Shapley Values, Interaction Index, Post-hoc}}
}

\begin{abstract}
    \shap is a popular approach to explain black-box models by revealing the importance of individual features.
    As it ignores feature interactions, \shap explanations can be confusing up to misleading. \nshap, on the other hand, reports the additive importance for all subsets of features. While this does include all interacting sets of features, it also leads to an exponentially sized, difficult to interpret explanation. In this paper, we propose to combine the best of these two worlds, by partitioning the features into parts that significantly interact, and use these parts to compose a succinct, interpretable, additive explanation.
    We derive a criterion by which to measure the representativeness of such a partition for a models behavior, traded off against the complexity of the resulting explanation.
    To efficiently find the best partition out of super-exponentially many, we show how to prune sub-optimal solutions using a statistical test, which not only improves runtime but also helps to detect spurious interactions. Experiments on synthetic and real world data show that our explanations are both more accurate resp. more easily interpretable than those of \shap and \nshap. %
\end{abstract}

\input{introduction}

\input{theory}

\input{algorithm}

\input{related}

\input{experiments}

\input{conclusion}

\input{impact_statement}

\bibliographystyle{icml2024}
\bibliography{bib/abbreviations,bib/bib-jilles,bib/bib-paper}

\newpage
\appendix
\onecolumn
\input{appendix}

\end{document}

%% file: defines.tex
\include{definitions}

\newcommand{\ourmethod}{i$\textsc{Shap}$\xspace}
\newcommand{\ourmethodE}{\ourmethod-Exact\xspace}
\newcommand{\ourmethodG}{\ourmethod-Greedy\xspace}
\newcommand{\shap}{$\textsc{Shap}$\xspace}

\newcommand{\lime}{$\textsc{Lime}$\xspace}
\newcommand{\nshap}{$\textsc{nShap}$\xspace}
\newcommand{\Do}{\textit{do}}

\newenvironment{proof}{\paragraph{Proof:}\itshape}{\hfill$\square$ \\}

\def\E{{\mathrm{E}}\,}

\def\P{{\Pi}}

\makeatletter
\renewcommand*{\@fnsymbol}[1]{\ensuremath{\ifcase#1\or   \circ\or \bullet\or *\or \ddagger\or
		\mathsection\or \mathparagraph\or \|\or **\or \dagger\dagger
		\or \ddagger\ddagger \else\@ctrerr\fi}}
\makeatother

\newcounter{inlineequation}
\definecolor{ind-effect}{HTML}{1B77B8}
\definecolor{pos-effect}{HTML}{25A012}
\definecolor{neg-effect}{HTML}{E05263}

\pgfplotscreateplotcyclelist{ishap-stacked-ybar}{%
	{ind-effect, fill=ind-effect},
	{pos-effect, fill=pos-effect},	
	{neg-effect, fill=neg-effect},
}

\definecolor{n1}{HTML}{1B77B8}
\definecolor{n2}{HTML}{FF7E00}
\definecolor{n3}{HTML}{25A012}
\definecolor{n4}{HTML}{D82520}
\definecolor{n5}{HTML}{9567C1}
\definecolor{n6}{HTML}{8D5649}
\definecolor{n7}{HTML}{E476C5}
\definecolor{n8}{HTML}{7F7F7F}
\definecolor{n9}{HTML}{BCBD00}
\definecolor{n10}{HTML}{00BED1}

\pgfplotscreateplotcyclelist{nshap-double}{%
	{n1, fill=n1},
	{n1, fill=n1},
	{n2, fill=n2},
	{n2, fill=n2},
	{n3, fill=n3},
	{n3, fill=n3},
	{n4, fill=n4},
	{n4, fill=n4},
	{n5, fill=n5},
	{n5, fill=n5},
	{n6, fill=n6},
	{n6, fill=n6},
	{n7, fill=n7},
	{n7, fill=n7},
	{n8, fill=n8},
	{n8, fill=n8},
	{n9, fill=n9},
	{n9, fill=n9},
	{n10, fill=n10},
	{n10, fill=n10},
}

\pgfplotsset{
	my legend/.style={
		mark=square*,
		fill,
		mark options={scale=4, fill opacity=0.5}
	}
}

\pgfplotsset{
	my boxplot/.style={
		eda boxplot,
		boxplot/draw direction=y,
		boxplot={
			box extend=0.2
		},
		height = 3.2cm,
		width = \linewidth,
		x label style 		= {at={($(axis description cs:0.5,0) + (0,-6pt)$)}, font=\scriptsize},
		y label style 		= {at={($(axis description cs:0,0.5) + (-16pt,5pt)$)}, font=\scriptsize},
		ymax = 1,
        ytick={0,0.2,0.4,0.6,0.8,1},
		outer sep=0pt,
	}
}

\pgfplotsset{
	my box/.style={
		solid,
		mark=*,
		mark options={scale=0.5},
		fill opacity=0.4,
		fill
	}
}

%% file: definitions.tex
\newcommand{\interaction}{\mathcal{I}}
\newcommand{\ind}{\perp\!\!\!\!\perp}

%% file: introduction.tex
\section{Introduction}

Decision processes must be fair and transparent regardless of whether it is driven by a human or an algorithm. 
Post-hoc explainability methods offer a solution as they can generate explanations that are independent of the underlying model $f$, and are hence also applicable to powerful black-box machine learning models.
One of the most popular post-hoc approaches is \shap \citep{lundberg:17:shap}, which provides intuitive explanations for a decision $f(x)$ 
of an arbitrary model $f$ for an individual $x$ in terms of how much a specific input value $x_i$ contributes to the outcome $f(x)$.
Additive explanations over \textit{single features} are succinct and easily understandable, but, only reliable when the underlying model is indeed additive. 
Whenever there are interactions between features in the model this can lead to misleading results~\citep{gosiewska:19:nottrustadditive}. 

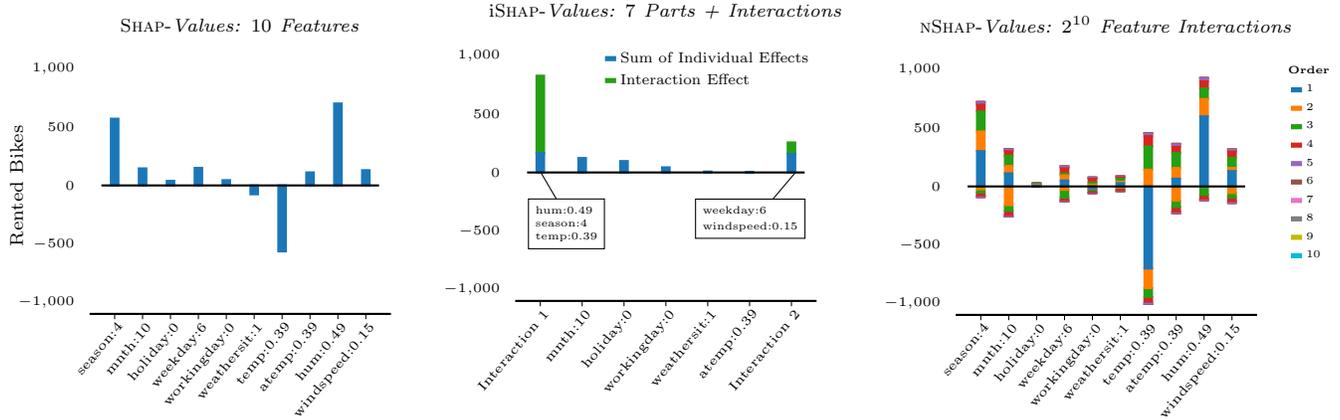
\begin{figure*}[t]
    \begin{minipage}[]{0.32\textwidth}
        \centering
        \input{figs/shap-bike.tex}
    \end{minipage}
    \hfill
    \begin{minipage}[]{0.32\textwidth}
        \centering
        \vspace{-0.1cm}
        \input{figs/ishap-bike.tex}
    \end{minipage}
    \hfill
\begin{minipage}[]{0.32\textwidth}
	\centering
    \vspace{0.4cm}
	\input{figs/nshap-bike.tex}
\end{minipage}
    \vspace{-0.6cm}
    \caption{Comparison of \shap (left), our approach \ourmethod (middle) and \nshap (right) on the Bike Sharing dataset \citep{fanaee:12:bikedataset}.
    \shap does not inform about any interactions, while \nshap overflows with information. \ourmethod provides a concise explanation for the high predicted demand:
    its is a warm and dry day for winter (\texttt{Season:4}, \texttt{Hum:0.49} and \texttt{Temp:0.39}) and a Saturday with little wind (\texttt{Weekday}:6 and \texttt{Windspeed}:0.15).}
    \label{fig:comparison}
\end{figure*}

Interaction index explanations \citep{lundberg:20:treeshap,sundararajan:20:shapleyinteraction,tsai:23:faithshap} 
address this weakness by considering \textit{groups of features} $x_S$ with an non-additive effect on the prediction $f(x)$.
\nshap \citep{bordt:23:shapleygam} is a recent interaction-based approach that decomposes a prediction $f(x)$ into a generalized additive model $\sum_{S \subset [d]} \Phi_S^n$,
where each combination of features, of which there are exponentially many, is assigned an additive contribution. Although this permits modeling any interaction, the size of the explanation makes it arduous to compute, 
and even harder to interpret and make use of.

To illustrate, we show the explanations of \shap, \nshap, and our method, \ourmethod for a bike rental prediction. On the left of Fig.~\ref{fig:comparison} we give the explanation that \shap provides. It is easy to understand, but counterintuitive: winter and humidity are listed as beneficial, while medium temperature is identified as a negative factor for bike rentals. 
On the right, we show the \nshap explanation for all 1024 subsets of features.
The simplicity of \shap is exchanged for an overflow of detail;
for example, \texttt{Temperature} has now equally strong positive and negative $n^{\mathit{th}}$ order interactions and it is hard to say which are truly important.
In fact, behind the visually appealing explanation there is a total of 751 non-zero \nshap values to be considered.

In the middle, we show our proposed \ourmethod explanation. It partitions the feature set into parts that significantly interact, and provides an additive explanation over these.
\ourmethod immediately reveals that there are two interactions responsible for the high predicted demand:
it is a dry and relatively warm day for winter (\texttt{Season:4}, \texttt{Hum:0.49} and \texttt{Temp:0.39}) and a Saturday with little wind (\texttt{Weekday:6} and \texttt{Windspeed:0.15}).

In the following, we outline the theory and algorithm behind \ourmethod to combine the best of both worlds: 
we propose to partition the feature set into parts that interact strongly, and use them to
compose an interpretable, additive explanation.
We formalize this by deriving an objective function for the ideal partition of an additive explanation.
The objective is inspired by game theory and seeks to find those coalitions of players, which together best approximate the full game as an additive function.
The main hurdle is on the computational side due to the combinatorial explosion of the number of possible partitions.
To this end, we devise a statistical test to prune the search space by testing pairwise for significant interactions.
We then show how to construct an additive explanation from the partition, and test them on various benchmarks against both additive and interaction based explanations.

%% file: figs/shap-bike.tex
\begin{tikzpicture}
    \begin{axis}[	
        pretty ybar stacked,
		height          = 5cm,
		width           = \linewidth,
		bar width       = .1cm,
		ymin = -1100,
		ymax = 1100,
		symbolic x coords = {season,mnth,holiday,weekday,workingday,weathersit,temp,atemp,hum,windspeed},
		xticklabel style={anchor= north east, rotate=50, yshift=0.2em, xshift=0.3em, font=\tiny},
		xticklabels = {season:4,mnth:10,holiday:0,weekday:6,workingday:0,weathersit:1,temp:0.39,atemp:0.39,hum:0.49,windspeed:0.15},
		yticklabel style={font=\tiny},
		ylabel style={yshift=1.6ex},
		cycle list name = nshap-double,
		ylabel          = {Rented Bikes},
		title ={\scriptsize \shap-\emph{Values: $10$ Features}},
		grid=none,
		legend style    = {
			at      = {(0.5,-0.20)},
			anchor  = north,
			legend columns=-1},
		]
		\draw[shorten >=-5pt, shorten <=-5pt] (axis cs:season,0) --  (axis cs:windspeed,0);
        \pgfplotsinvokeforeach{1}{ 
			\addplot table[y=Individual Effect, x=Features, col sep = semicolon] {expres/bike-shap.csv}; 
        }
	\end{axis}    
\end{tikzpicture}

%% file: figs/ishap-bike.tex
\begin{tikzpicture}
    \begin{axis}[	
        pretty ybar stacked,
		height          = 5cm,
		width           = \linewidth,
		bar width       = .1cm,
		ymin = -1100,
		ymax = 1100,
		symbolic x coords = {y1,y3,y4,y5,y6,y7,y2},
		xticklabels     = {holiday:0.00,mnth:10,atemp,weathersit,workingday},
		xticklabels = {Interaction 1,Interaction 2,mnth:10,holiday:0,workingday:0,weathersit:1,atemp:0.39},
		xticklabel style={anchor= north east, rotate=50, yshift=0.2em, xshift=0.3em, font=\tiny, align=right},
		cycle list name = ishap-stacked-ybar,
		yticklabel style={font=\tiny},
		title ={\scriptsize \ourmethod-\emph{Values: $7$ Parts + Interactions}},
		grid=none,
		legend style    = {
			font			= {\scriptsize},
			nodes			= {scale=0.8},
			cells			= {scale=0.2},
			draw=none, fill=none,
			at      = {(1,1)},
			anchor  = north east,
			legend columns=1},
		]
		\draw[shorten >=-5pt, shorten <=-5pt] (axis cs:y1,0) --  (axis cs:y2,0);
		\addlegendimage{my legend,ind-effect}
		\addlegendentry{Sum of Individual Effects}
		\addlegendimage{my legend,pos-effect}
		\addlegendentry{Interaction Effect}
        \pgfplotsinvokeforeach{1,...,3}{ 
            \addplot table[y=x#1, x=y, col sep = semicolon, on layer=background layer] {expres/bike-ishap.csv}; 
        }
		\node [draw,fill=white, scale=0.8, thin] at (rel axis cs: 0.17,0.3) {\tiny
			\shortstack[l]{
			hum:0.49\\
			season:4 \\
			temp:0.39}
		};
		\draw[thin] (rel axis cs:0.17,0.33) -- (rel axis cs:0.085,0.5);
		\node [draw,fill=white, scale=0.8, thin] at (rel axis cs: 0.78,0.32) {\tiny
			\shortstack[l]{
			weekday:6\\
			windspeed:0.15
			}
			};
		\draw[thin] (rel axis cs:0.8,0.32) -- (rel axis cs:0.93,0.5);
	\end{axis}
\end{tikzpicture}

%% file: figs/nshap-bike.tex
\begin{tikzpicture}
    \begin{axis}[	
        pretty ybar stacked,
		height          = 5cm,
		width           = \linewidth,
		bar width       = .1cm,
		ymin = -1100,
		ymax = 1100,
		symbolic x coords = {season,mnth,holiday,weekday,workingday,weathersit,temp,atemp,hum,windspeed},
		xticklabels = {season:4,mnth:10,holiday:0,weekday:6,workingday:0,weathersit:1,temp:0.39,atemp:0.39,hum:0.49,windspeed:0.15},
		xticklabel style={anchor= north east, rotate=50, yshift=0.2em, xshift=0.3em, font=\tiny},
		cycle list name = nshap-double,
		yticklabel style={font=\tiny},
		title ={\scriptsize \nshap-\emph{Values: $2^{10}$ Feature Interactions}},
		grid=none,
		legend style    = {
			font			= {\scriptsize},
			nodes			= {scale=0.8},
			cells			= {scale=0.2},
			draw=none, fill=none,
			at      = {(1.1,1)},
			anchor  = north west,
			legend columns=1},
		]
		\draw[shorten >=-5pt, shorten <=-5pt] (axis cs:season,0) --  (axis cs:windspeed,0);
		\addlegendimage{empty legend}
		\addlegendentry{\hspace{-.3cm}\tiny\textbf{Order}}
		\addlegendimage{my legend,n1}
		\addlegendentry{\tiny 1}
		\addlegendimage{my legend,n2}
		\addlegendentry{\tiny 2}
		\addlegendimage{my legend,n3}
		\addlegendentry{\tiny 3}
		\addlegendimage{my legend,n4}
		\addlegendentry{\tiny 4}
		\addlegendimage{my legend,n5}
		\addlegendentry{\tiny 5}
		\addlegendimage{my legend,n6}
		\addlegendentry{\tiny 6}
		\addlegendimage{my legend,n7}
		\addlegendentry{\tiny 7}
		\addlegendimage{my legend,n8}
		\addlegendentry{\tiny 8}
		\addlegendimage{my legend,n9}
		\addlegendentry{\tiny 9}
		\addlegendimage{my legend,n10}
		\addlegendentry{\tiny 10}

        \pgfplotsinvokeforeach{1,...,10}{ 
			\addplot+ table[y=Order#1-pos, x=Features, col sep = semicolon] {expres/bike-nshap.csv}; 
			\addplot table[y=Order#1-neg, x=Features, col sep = semicolon] {expres/bike-nshap.csv}; 
        }

	\end{axis}    
\end{tikzpicture}

%% file: theory.tex
\section{Theory}
\label{sec:theory}
We consider a machine learning model $f: \mathcal{X} \to \mathbb{R}$ with a domain $\mathcal{X}$ over $d$ univariate random variables
$X_1$ to $X_d$.
We denote a subset of input variables by $X_S$, where $S$ is the index set and denote the set of all indices as $[d] = \{1,...,d\}$.
We define an explanation as a set of tuples $\{(S_i,e_i)\}$ where $S_i$ is an index set with an 
explanatory value $e_{i}$ to the prediction $f(x)$.
For example, \shap explains using singletons $S_i = \{i\}$, where $e_i$ are the Shapley values.
\nshap on the other hand uses all sets $S_i \subset [d]$ from the power set of features.

We view the local prediction $f(x)$ as a coalition game, where a set of players $x_S$ receive a payoff $v(S)$ defined by a value function $v : 2^d \to \mathbb{R}$,
which we then analyze to determine the contribution of each coalition, made up of individual features, to the prediction.
W.l.o.g.~we assume that the value function is normalized, i.e.~$v(\emptyset) = 0$, which can be achieved by pre-processing $f$ so that $E[f(X)]=0$.
In the context of machine learning, two main variants of value functions are used: the observational value function 
$v(S;f,x) = \E\left[f(X)|X_S=x_S\right]$ \citep{lundberg:17:shap}, and the interventional value function $v(S;f,x) = \E\left[f(X)|\Do(X_S=x_S)\right]$ \citep{janzing:20:causalshapley}.
Our method is based directly on $v$, which can then be instantiated with either distribution.
In the following, we omit the specific instance of $f$ and $x$ and refer to the value function simply as $v(S)$. 
\subsection{Objective}
Our goal is to construct an explanation $\{(S_i,e_i)\}$ for $f(x)$ that has a succinct and thus interpretable amount of components $S_i$,
and where the additive interpretation $\sum_i e_i$ approximates the behavior of $f$ best.
For example, if $f$ is a linear model, then the value function of a single feature $i$ is the weight $w_i$ times the deviation from the mean, i.e.~$v(i)=w_i(x_i - \E[X_i])$,
and the value function of a coalition of features $S$ is the sum of the individual value functions $v(S) = \sum_{i \in S} v(i)$.
For complex models such as neural networks however, there exists no exact analytic decomposition in practice.

Instead, we focus on finding a partition $\P$ of the features space $[d]$ such that each feature $i$ is contained in only one set $S_j$.
Each feature $i \in S_j$ is associated with only one explanatory value $e_j$, whilst still informing about interactions between features in $S_j$.
In particular, we want to find that partition $\P$ that $\min \left(f(x) - \sum_{S\in\P}v(S)\right)^2$, i.e.~the partition $\P$ that \textit{approximates} $f$ of $x$ best. 
The general idea is that if feature $x_i$ and $x_j$ interact and have a large joint effect on the result, the local surrogate model will make a large mistake if $x_i$ and $x_j$ are not in the same set $S$.

It is easy to see that the objective is trivially minimized by $\P = \{[d]\}$, 
which would not give any insight into the inner workings of the model.
Therefore, we regularize the complexity of the explanation
by penalizing the number of allowed interactions through the $L$0 norm.
For each set $S \in \P$, there are $|S||S-1|/2$ pairwise interactions, making the full $L$0 norm of a partition
$\mathcal{R}(\P) = \sum_{S \in \P} |S||S-1|/2$.
Thus, we define the optimal partition $\P^*$ in regards the value function $v$ of an algorithmic decision $f(x)$
as the partition $\P$ which minimizes
\begin{align}
	\P^* & = \arg \min_{\P} \left( f(x)- \sum_{S_i \in \P} v(S_i) \right)^2 \\
	& + \lambda \cdot \sum_{S_i \in \P} \frac{1}{2}|S_i||S_i-1| \quad .
	\label{eq:objective}
\end{align}
The explanation $\{(S_i,e_i)\}$ is then constructed from the sets $S_i \in \P^*$ of the optimal partition.

\subsection{Partitioning}
Objective \eqref{eq:objective} poses a challenging optimization problem.
Finding the best partition is a constrained variant of the subset sum problem and thus NP-hard (see Supplement \ref{ap:complexity}).
The number of partitions for a set of $d$ features is the Bell number $B_d$ that grows super-exponentially with $d$, ruling out exhaustive search.
Approximate solutions for subset sum problems can be computed in pseudo-polynomial time \citep{pisinger:99:knapsack},
they however require the value function to be computed for all elements of the power set, of which there are exponentially many.

Instead, our approach initially looks at all pairs of variables $x_i$ and $x_j$, and seeks to determine whether they have an interactive, non-additive effect on the prediction $f(x)$.
If there is no interaction, we can rule out the combination of $x_i$ and $x_j$ in the optimal partition, and thus prune the search space drastically.
\begin{restatable}{definition}{definitioninteraction}{
	\label{def:interaction}
	Given a value function $v$, the interaction $\interaction$ between $x_i$ and $x_j$ in the context of $x_S$ is defined as 
	\[
		\interaction(i,j,S) = v(S \cup i) + v(S \cup j) - v(S \cup \{i,j\}) - v(S)\;.
	\]
}
\end{restatable}
This definition of interaction, as introduced in \cite{lundberg:20:treeshap}, measures the effect of setting $X_i=x_i$ and $X_j=x_j$ individually, in contrast to the combined effect, whilst accounting 
for a covariate set $x_S$.
We now show, that if for any covariate set $S$, there is no interaction between $x_i$ and $x_j$, then $i$ and $j$ are not be grouped in the optimal partition
with regard to Objective \eqref{eq:objective}.
To this end, we begin by showing that the additivity of effects for a pair $x_i$ and $x_j$ is a sufficient criterion to rule out their pairing, 
and then show in which cases we can use the absence of interaction as an indicator for additivity.
\begin{restatable}{theorem}{theoremnonadditive}{
	\label{eq:noninteractive}
	Let $v$ be additive for the variables $x_i$ and $x_j$, so that for all covariates $S \subseteq [d]\setminus \{i,j\}$ there exists a partition $A \cup B = S$ with
	\[
		v(A\cup i) + v(B\cup j) = v(S\cup \{i,j\})\;.
	\]
	Then, $x_i$ and $x_j$ do not occur together in the optimal partition $\P^*$ in regards to Objective \eqref{eq:objective},
	i.e.
	\[
		\nexists S_k \in \P^*: i \in S_k \land j \in S_k\;.
	\]
}
\end{restatable}
\begin{proof}
	Assume the optimal partition $\P^*$ contains a set $S$ where $i, j \in S$.
	Then, the value function $v(S)$ is decomposable into $v(S) = v(A \cup i)+v(B \cup j)$. 
	Thus, we may construct a partition $\P'$ with $A \cup i$ and $B \cup j$, where the reconstruction error
	$f(x) -\sum_{S_i \in \P'} v(S_i)$ remains the same and
	its regularization penalty shrinks, i.e.~$R(\P^*)>R(\P')$.
	It follows that the overall the objective of the partition $\P'$ is lower than $\P^*$, contradicting its optimality.
	\vspace{0.2cm}
\end{proof}
Theorem \ref{eq:noninteractive} confirms the intuition that if $v$ is additive for two variables $x_i$ and $x_j$, then they do not occur as part of the same set in the optimal partition.
The main challenge lies in the exponential quantity of contexts $S$ to consider, where the effect of $x_i$ and $x_j$ may differ, which is the same underlying problem
hampering the interpretability and computability of \nshap.
Instead, we show how to reduce this effort to only a single test per pair $x_i$ and $x_j$, which under some mild assumption allows to detect an absence 
of interaction and hence rule out their pairing from the optimal solution.
The resulting explanations are tractable in real time and more comprehensible due to less overall components.
\begin{restatable}{assumption}{additivityassumption}
	\label{as:additivity-assumption}
	If $v$ is additive for a partition $A, B$ of $S$, i.e.~$v(S) = v(A) + v(B)$,
	then it is also additive for all subsets $A' \subseteq A, B' \subseteq B$, so that
	\[
		\forall A' \subseteq A, B' \subseteq B: v(A') + v(B') = v(A' \cup B')\; .
	\]
\end{restatable}
Assumption \ref{as:additivity-assumption} requires that the additivity of two sets of features $A$ and $B$ is preserved for all of their subsets.
For example, if we find that $v(\{x_1,x_2,x_3\}) = v(\{x_1,x_2\}) + v(\{x_3\})$, then we also assume that $v(\{x_1,x_3\}) = v(\{x_1\}) + v(\{x_3\})$.
This holds for many popular value functions, including the interventional value function by \citet{janzing:20:causalshapley} and the original observational value 
function used by \citet{lundberg:17:shap} in conjunction with an underlying additive function $f$,
but does not generally hold for Asymmetric Shapley Values \citep{frye:20:asymmetricshapley} (see Supplement \ref{ap:additivity}).

\begin{restatable}{theorem}{theoremtest}
	\label{eq:theorem-test}
	If the expected interaction of a pair of variables $i$ and $j$ is not zero, i.e.
	\begin{gather}
		\label{eq:test}
		\E_S\left[\interaction(i,j,S)\right] \neq 0\;,
	\end{gather}
	then $v$ is not additive for $i$ and $j$, i.e.
	\begin{gather}
		\label{eq:nonadditive}
		\exists S \subseteq [d]\setminus \{i,j\}: \forall A, B, A \cup B = S, A\cap B = \emptyset:\\ 
		v(A \cup i) + v(B \cup j) \neq v(S \cup i, j)\;.
	\end{gather}
\end{restatable}
\begin{proof}
If there is interaction between $i$ and $j$, we show that there exists a covariate set $S$ for which $v$ is not additive for $i$ and $j$.
First, we note that 
\begin{gather}
	\E_S\left[\interaction(i,j,S)\right] \neq 0 \\
   \implies  \exists S \subseteq [d]\setminus \{i,j\}: \interaction(i,j,S) \neq 0 \;,
\end{gather}
i.e.~there exists a covariate set $S$ for which the interaction is not zero.
For this set $S$, it holds that 
\begin{gather}
	v(S \cup i) + v(S \cup j) \neq v(S \cup i,j) + v(S) \;. \label{eq:nonadditive-proof}
\end{gather}
If $v$ indeed was additive for $i$ and $j$, then for $S$ there exists a partition $A \cup B = S$ so that
\[
	v(S \cup i, j) = v(A \cup i) + v(B \cup j)\;.	
\]
By Assumption \ref{as:additivity-assumption}, we know that this decomposition also holds for $S$, $S \cup i$ and $S \cup j$, so that we can rewrite Equation \eqref{eq:nonadditive-proof} as
\begin{gather}
	v(A \cup i) + v(B) + v(A) + v(B \cup j) \\
	\neq v(A \cup i) + v(B \cup j) + v(A) + v(B)\;.
\end{gather}
This statement is a contradiction, and thus proves that $v$ is not additive for $i$ and $j$.
\end{proof}

With Theorem \ref{eq:theorem-test}, we show that we can reject the additivity of a pair of variables $i$ and $j$ if their interaction is non-zero.
As per Theorem \ref{eq:noninteractive}, any pair of variables $i$ and $j$ that is additive is not grouped together in the optimal partition.
Thus, to obtain the optimal partition $\Pi^*$, we need to only consider all interacting pairs of variables $i$ and $j$.

Considering these pairs alone however does not suffice. Let $i$ and $j$ be non-additive, and let $j$ and $k$ be non-additive too, then
we can show that $i$ and $k$ are also non-additive, i.e.~potentially grouped together in the optimal partition.
\begin{restatable}{theorem}{theoremmultiple}
	Let $v$ be non-additive for $i$ and $j$, i.e.
	\begin{gather}
	\exists S_1 \subseteq [d]\setminus\{i,j\}: \forall A, B, A \cup B = S_1, A\cap B = \emptyset:\\
	v(A \cup i) + v(B \cup j) \neq v(S_1 \cup i, j)
	\end{gather}
	and let $v$ be non-additive for $j$ and $k$, i.e.~$\exists S_2 \subseteq [d]\setminus\{j,k\}: \forall A,B: v(A \cup j) + v(B \cup k) \neq v(S_2 \cup j, k)$.
	Then $v$ is also not additive for the variables $i$ and $k$.
	\label{eq:theorem-multiple}
\end{restatable}

We provide the full proof in the Supplement \ref{ap:proofs}.
Theorem \ref{eq:theorem-multiple} allows us to reject the additivity of a pair of variables $i$ and $j$ if they are connected by a path of interactions.
This helps us to reduce the search space so that the partition contains only those sets where all variables are interconnected.
In practice, we run the risk of falsely eliminating pairs of variables where their effect is inconclusive, i.e.~the interaction is zero, and there is 
no path between them.
Our evaluation however indicates that requiring significant interactions has a profusely positive effect by
preventing the fit spurious interactions, whilst allowing the majority of true interactions to go through.
\subsection{Explanation}
Once the optimal partition $\P^*$ is obtained, we use the discovered interacting sets $S_i \in \P^*$ as building blocks to explain the algorithmic decision $f(x)$.
We set the contribution $e_i$ of each feature set $S_i$ to the Shapley values of a new game $v'$, where the players are made up by the sets $S_i$ instead.
This new game $v'$ allows to only include either all or no features of a feature set $S_i$, and takes the values of the original value function $v$.
As a result, we return for each algorithmic decision $f(x)$ based on the optimal partition an additive explanation $\{(S_i,e_i)\}_{S_i \in \P^*}$, where $\sum_i e_i =f(x)$.
Finally, we quantify the amount of interaction in $S_i$ using Definition \ref{def:interaction} and extend it onto sets as the difference between their joint 
contribution and the contribution of grouped features individually $v(S_i)-\sum_{j\in S_i}v(j)$.

%% file: algorithm.tex
\section{Algorithm}
The number of possible partitions of a set of $d$ variables is the Bell number, $B_d$. 
The total amount grows super-exponentially with the number of variables $d$, making it vital to restrict the search space.
To this end, we introduce the \ourmethod algorithm for Interaction-Aware Shapley Value Explanations.
It enables us to find the optimal partition $\P^*$ which minimizes the regularized reconstruction error of Objective \ref{eq:objective}.

\subsection{Pairwise Interaction Test}
Before we can find the optimal partition, we first identify all pairs of variables that show significant interaction.
As per Theorem \ref{eq:theorem-test}, for a pair $x_i$ and $x_j$, we formulate the null-hypothesis of no interaction as
\[
    H_0: \E_S\left[\interaction(i,j,S)\right] = 0\; .    
\]
We use a $t$-test to find statistically significant interaction effects under a user specified significance level $\alpha$.
We construct an undirected graph with a node for each feature, and draw an edge wherever $H_0$ is rejected.
Figure \ref{fig:interaction-graph} shows an example of such an interaction graph.
Overall, the graph provides insight into the significant interactions which occur in the model and
supplements the additive \ourmethod explanation.
With it, we can now run a much accelerated graph partitioning algorithm to find the best partitioning.

By Theorem \ref{eq:theorem-multiple}, we know that any pair of nodes which is connected by a path in the interaction graph is not additive. 
Therefore, the optimal partition consists only of connected components of the interaction graph, and their subsets.
This because for any pair of variables $i$ and $j$ which are not connected by a path, and therefore not additive, 
a partition $\P$ which contains both $i$ and $j$ is suboptimal as per Theorem \ref{eq:noninteractive}.

\begin{figure}[t]
    \centering
    \includegraphics[width=0.8\linewidth]{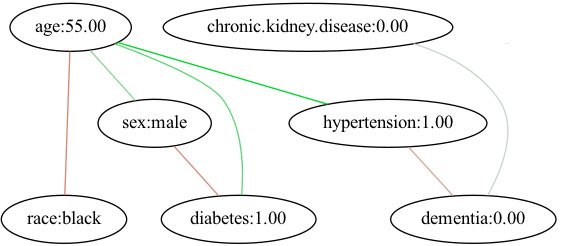}
    \caption{Interaction graph for predicted survival of a hospitalized COVID-19 patient.
    The detrimental effect of diabetes and hypertension on survival is alleviated by the relatively young age (55) of the patient.
    }
    \label{fig:interaction-graph}
\end{figure}

\subsection{Search}
Given the interaction graph we can derive all valid candidate partitions. A partition is valid if for each component $S$ and all pairwise features $x_i, x_j \in S$ there exist a path between $x_i$ and $x_j$ in the interaction graph. In the naive exhaustive approach we test all valid partitions, and select the one which minimizes our objective. 
We evaluate all eligible partitions $\P$ by sampling the value function $v(S)$ for each subset of features $S \in \P$,
computing the reconstruction error $f(x)- \sum_{S \in \P} v(S)$ and add the regularization penalty $\mathcal{R}(\P)$.
To avoid recomputing the value function for the same subset in different partitions, we additionally buffer the value function $v(S)$
for re-use in other partitions. \ourmethodE does an exact search over all eligible candidates and is guaranteed to find the optimal partition,
but comes to its limit for many variables, for which we introduce a greedy search variant.
\ourmethodG is a bottom up approach, which starts with the all singleton partition $\P_0 = \{ S_1=\{x_1\}, \dots, S_d=\{x_d\}\}$. 
Iteratively, the two sets $S_i$ and $S_j$ which yield the highest gain in the objective are merged.
Here, only merges are considered which do not violate the interaction graph, i.e. $S_i$ and $S_j$ are connected by an path. Naturally, the greedy approach is not guaranteed to be optimal, but as we will see in the evaluation, achieves near optimal results in practice. 
The pseudocode for both the greedy and exact variant and a complexity analysis can be found in the Supplement \ref{ap:algorithm}.

%% file: related.tex
\section{Related Work}
We focus on post-hoc, model-agnostic explainability approaches \citep{ribeiro:16:modelagnostic} that treat the model $f$ as a black-box and generate explanations by perturbing the input and analyzing the output.
Here, explanations can be categorized into global explanations of $f$ and local explanations of a particular decision $f(x)$.

\citet{friedman:01:boosting} introduced the partial dependence plot (PD), a global explanation method that visualizes the relationship between a variable $X_i$ and the predicted output $f(X)$. 
PD plots are well suited for an injective relationship between $X_i$ and $f(X)$, however, unlike our method, are not well suited for cases where interaction effects between more than two variables occur.
Functional ANOVA (analysis of variance) \citep{hooker:04:anova,hooker:07:func-anova} is another global explanation approach which aims at discovering non-additive interactions between input variables.
\citet{sivill:23:shapley-sets} propose to discover interacting feature sets for a given model, one step further \citet{herbinger:23:decomposing} aim to partition the feature space into subspaces by minimizing feature interactions. 
Overall, global model explanations can give a good overview over a model $f$, but produce non-conclusive explanations when faced with complex, highly interactive functions where it is hard to create a single global summary.

Local explanations on the other hand aim to explain the decision of a model $f(x)$ for a particular instance $x$.
Local Interpretable Model Agnostic Explanations (\lime) \citep{ribeiro:16:lime}, is perhaps the most influential post-hoc explainability method.
\lime explains a prediction $f(x)$ by constructing a local surrogate model $f'$ that is interpretable, for example a Linear Regression or a Decision Tree. 
\lime is model agnostic and generates simple, intuitive explanations, but has a fuzzy data sampling process and no guarantees as it is based on purely heuristics. As \lime fits a local model, the resulting explanations can be misleading \citep{ribeiro:18:anchors} or non informative for cases where $x$ is \emph{far} from any decision boundary. 

\input{figs/synthetic-experiments}

A different style of explanation are approaches which explain a decision $f(x)$ through combinations of features $x_i$ \citep{carter:19:sis, ribeiro:18:anchors}. In essence both method find a sufficient set of variables such that the prediction $f(x)$ does not change. In contrast to our method they can not provide explanations for regression models and they do not explain which combinations of variables has which effect on the prediction. 
Counterfactual explanations \citep{wachter:17:counterfactual, karimi:20:counterfactual,van:21:counterfactual,mothilal:20:counterfactual,poyiadzi:20:counterfactual} are another type of local explanations that bring together causality and explainability.
Algorithmic Recourse \citep{joshi:19:recourse, karimi:21:recourse} goes one step further and wants to find the best set of changes to reverse a models decision. 
These methods are most appropriate if the user seeks actionable insights, i.e.~how to reverse a decision, and, in contrast, do not explain a prediction. 

One of the most widely used approaches for explainable AI are Shapley values \citep{lundberg:17:shap}.
Shapley values where originally introduced \citep{shapley:53:shapley} in game theory to measure the contributions of individual players. 
Recently different variants have been proposed like asymmetric \cite {frye:20:asymmetricshapley} and causal Shapley values \citep{heskes:20:causalshap}. Unlike our method they provide per feature attributions and no further insight into which features interact. 
Describing interaction is getting more attention in recent years. \citet{jung:22:measuring} propose to quantify the effect of a group of causes 
through $\Do$-interventions, but focus on estimating those effects from non-experimental data.
\citet{jullum:21:groupshapley} propose to compute Shapley values on predefined feature sets, where they suggest to either group semantically related features or correlated features.
Unlike our approach the feature sets are not discovered and independent of the model. %

Most closely related is a recent class of so called \emph{interaction index} explanations.
The work by \citet{lundberg:20:treeshap} introduces \emph{SHAP interaction values} extending SHAP explanations to all pairwise interactions. 
\citet{sundararajan:20:shapleyinteraction} and \citet{tsai:23:faithshap} both derive Shapley interaction indices for binary features that cover the entire powerset.
Most recently \nshap was introduced \citep{bordt:23:shapleygam}, which extends Shapley interaction values to sets of degree $n$.
Similar to our method, \nshap explains a decision $f(x)$ through a generalized additive model $\sum_{S \subset [d]} \Phi_S^n$.
The main difference is that \nshap provides a value for the entire powerset of features, whereas 
our method chooses a succinct representation selecting interacting components. %

%% file: figs/synthetic-experiments.tex
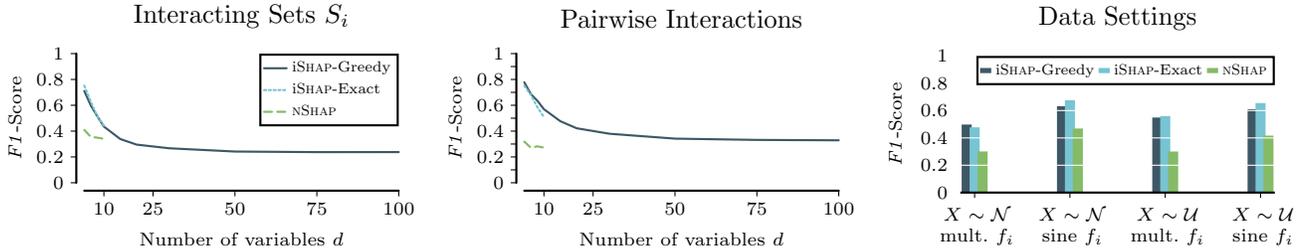
\begin{figure*}[t]
    \begin{minipage}[t]{0.33\linewidth}
        \centering
        \begin{tikzpicture}
        \begin{axis}[
        pretty line,
        cycle list name = prcl-line,
        width = \textwidth,
        height = 3.3cm,
        xtick = {5,10,...,20},
        xtick = {10,25,50,75,100},
        ymax = 1,
        ymin = 0,
        xlabel = {Number of variables $d$},
        ylabel = {$\mathit{F1}$-Score},
        title = {Interacting Sets $S_i$},
        legend entries= {\ourmethodG, \ourmethodE, \nshap},
        legend columns = 1,
        legend style = {at={(1,1.0)},fill=white,draw=black,opacity=1,font=\scriptsize,nodes			= {scale=0.8},
        cells			= {scale=1}},
        thick
        ]
        \addplot table[x=Iteration, y=F1-GREEDY, col sep=comma] {expres/synthetic/ishap.csv}; %
        \addplot table[x=Iteration, y=F1-FULL, col sep=comma] {expres/synthetic/exhaustive.csv}; %
        \addplot table[x=Iteration, y=F1-NSHAP, col sep=comma] {expres/synthetic/exhaustive.csv};
        \end{axis}
        \end{tikzpicture}
     \end{minipage}
    \begin{minipage}[t]{0.33\linewidth}
        \centering
        \begin{tikzpicture}
        \begin{axis}[
        pretty line,
        cycle list name = prcl-line,
        width = \textwidth,
        height = 3.3cm,
        xtick = {10,25,50,75,100},
        ytick = {0,0.1,...,1},
        yticklabels = {0,,0.2,,0.4,,0.6,,0.8,,1},
        ymax = 1,
        ymin = 0,
        xlabel = {Number of variables $d$},
        ylabel = {$\mathit{F1}$-Score},
        title = {Pairwise Interactions},
        thick
        ]
        \addplot table[x=Iteration, y=Pairwise-F1-GREEDY, col sep=comma] {expres/synthetic/ishap.csv}; %
        \addplot table[x=Iteration, y=Pairwise-F1-FULL, col sep=comma] {expres/synthetic/exhaustive.csv}; %
        \addplot table[x=Iteration, y=Pairwise-F1-NSHAP, col sep=comma] {expres/synthetic/exhaustive.csv};
        \end{axis}
        \end{tikzpicture}
    \end{minipage}
   \begin{minipage}[t]{0.33\linewidth}
        \centering
        \begin{tikzpicture}
            \begin{groupplot}[
                pretty ybar small,
                cycle list name = prcl-ybar,
                pretty labelshift,
                ymin = 0.,
                ylabel = {$\mathit{F1}$-Score},
                title = {Data Settings},
                ymax = 1,
                group style = { group size=3 by 1, horizontal sep=15pt},
                xlabel = {},
                symbolic x coords={normal-mult,normal-sin,uniform-mult,uniform-sin},
                xtick = {normal-mult,normal-sin,uniform-mult,uniform-sin},
                xticklabel style={align=center},
                legend entries= {\ourmethodG, \ourmethodE, \nshap},
                legend columns = 3,
                legend style = {at={(1,1)},fill=white,draw=black},
                xticklabels = {$X \sim \mathcal{N}$\\mult.~$f_i$,$X \sim \mathcal{N}$\\sine $f_i$,$X \sim \mathcal{U}$\\mult.~$f_i$,$X \sim \mathcal{U}$\\sine $f_i$}, ]
                \nextgroupplot[width=\textwidth,height=3.4cm]
                \addplot[c1,fill=c1] table[x=Config, y=iSHAP-Greedy,col sep=comma] {expres/synthetic/per_config.csv};
                \addplot[c2,fill=c2] table[x=Config, y=iSHAP-Exact,col sep=comma] {expres/synthetic/per_config.csv};
                \addplot[dollarbill,fill=dollarbill] table[x=Config, y=NSHAP,col sep=comma] {expres/synthetic/per_config.csv};
            \end{groupplot}
        \end{tikzpicture}
   \end{minipage}
   \caption{[Higher is better] $\mathit{F1}$ scores of recovered interactions in GAMs. \ourmethod is more accurate than \nshap in detecting 
   full sets of interacting features (left) and pairwise interactions (middle), and can do so on more features.
   \ourmethodG is equivalent to \ourmethodE and outperforms \nshap across all data settings(right).}
   \label{fig:experiments-acc}
\end{figure*}

%% file: experiments.tex
\section{Experiments}
In this section we empirically evaluate \ourmethod. We compare to \shap \citep{lundberg:17:shap} and \lime \citep{ribeiro:16:lime}, the most widely used additive explanation methods, and \nshap \citep{bordt:23:shapleygam}. We implemented \ourmethod in Python, and used the original Python implementation of \shap, \nshap and \lime. We provide the code and data generators in the Supplement. All experiments were conducted on a consumer-grade laptop.

\subsection{Discovering Interactions}
First, we examine whether \ourmethod recovers truly interacting sets of variables, and compare it against \nshap, which also 
informs about interactions.
To this end, we generate random generalized additive models $f$ (GAMs) for which we can define ground truth sets of interacting features $S_i$.
We sample $d$ feature variables $X_j$ from a uniform or a Gaussian distribution.
Then, we construct the ground truth partition $\P$
by repeatedly sampling sets $S_i$ of arbitrary size from a Poisson distribution.
Next, we define $f$ as
$f(x;\P) = \sum_{S_i \in \P} f_i(X_{S_i})\;,$
with a non-additive inner function $f_i$.
For the full details on the experimental setup we refer to Supplement~\ref{ap:experiments}.

Given a GAM $f(x;\P)$, we now compute for a random data point the \ourmethod and the \nshap explanation and evaluate the obtained feature partition $\hat{\P}$ (for \nshap, we take those sets with the highest interaction value).
We measure the quality of the predicted partition $\hat{\P}$ compared to the ground truth $\P$ using the
$\mathit{F1}$ score between the sets $S_i \in \P$ and $\hat{S}_i \in \hat{\P}$.
Additionally, we evaluate whether an explanation finds interactions on a pairwise level, where we compute the F1-score on paired variables $j,k \in S_i$ 
and $j', k' \in \hat{S}_i$.

We show the performance of \ourmethodG, \ourmethodE and \nshap in Fig.~\ref{fig:experiments-acc}.
For up to 10 variables, where \nshap terminates within the time limit of 10 min/explanation, 
\ourmethod is more accurate than \nshap in detecting full sets of interactions.
This advantage persists across all tested combinations of feature distributions $P(X)$ and classes of inner functions (Fig.~\ref{fig:experiments-acc} right).
Notably, \ourmethodG performs almost as good as \ourmethodE, showing the effectiveness of our interaction test in restricting the search space.
In Supplement \ref{ap:ablation}, we provide an ablation study showing how performance deteriorates without the interaction test.

In addition, we also evaluate the \emph{pairwise interaction} $\mathit{F1}$ score, as for an increasing amount of variables, the accuracy over sets deteriorates.
In particular, with ground truth $(\texttt{Age},\texttt{Sex},\texttt{Weight},\texttt{Height})$, finding $(\texttt{Age},\texttt{Sex},\texttt{Weight})$ gives no 
score on a set level, but contains half of the pairwise interactions.
We show the pairwise $\mathit{F1}$ score in the middle of Fig.~\ref{fig:experiments-acc}.
Here too, \ourmethod strongly outperforms \nshap and finds at least parts of the ground truth interacting sets, hence achieving a higher pairwise score.
Overall, \ourmethodG reliably uncovers most feature interactions and can do so given many variables.

\input{figs/real_world_experiments}

\subsection{Surrogate Model Accuracy}
Next, we evaluate the accuracy of \ourmethod, \shap, \nshap and \lime explanations as \textit{surrogate models} to explain different
classes of models.
Given a model $f$, this evaluation as described by \cite{poursabzi:21:evaluation} presents a user with a number of
explanations and then asks him to predict the models output on new, unseen data.
As all evaluated explanation methods have explicit, additive surrogate models,
we can directly compute the models output as implied by the resp.~explanation.
To this end, given two datapoints $x^{(1)}$ and $x^{(2)}$, we firstly construct a new datapoint $x'$ by taking features uniformly at random from $x^{(1)}$ or $x^{(2)}$.
Then, for each feature, depending on which datapoint it came from, we add its explained additive contribution to obtain the 
implied $\hat{f}(x')$ and compute the MSE to the true prediction $f(x')$ (full detail in Supplement \ref{ap:experiments}).

We consider five regression and two classification datasets\footnote{\textit{California}~\citep{pace:97:california}, \textit{Diabetes}~\citep{pedregosa:11:sklearn}, \textit{Insurance}~\citep{lantz:19:ml-r}, \textit{Life}~\citep{rajarshi:19:who}, \textit{Student}~\citep{cortez:08:student}, \textit{Breast Cancer}~\citep{Dua:19:uci} and \textit{Credit}~\citep{Dua:19:uci}} 
with increasingly complex, non-additive models $f$: linear models, support vector machines (SVM), random forests (RF), multi-layer perceptrons (MLP) and k-nearest neighbors (KNN).
We report the $R^2$ between the true model outcome and that predicted by the surrogate model in Fig.~\ref{fig:experiments-surrogate}.

For linear models $f$ the additivity assumption holds fully, which is reflected in the near perfect accuracy of all methods.
When using SVMs and random forests, the overall accuracy of all methods is decreased.
Here, \ourmethod emerges as the most accurate surrogate model with an $R^2$ over 0.9.
\lime struggles to model the local decision surface of an MLP as a linear model, and \nshap's exponential amount of modeled interactions of \nshap does not describe
k-nearest neighbors models well.
\shap does not take into account any interactions, leaving \ourmethod to consistently provide the most accurate surrogate model explanations.

From the performance across different datasets, shown in the middle of Fig.~\ref{fig:experiments-surrogate}, we see that \ourmethod outperforms \shap, \nshap and \lime on all datasets. 
In particular, we find that for the low-dimensional \textit{Diabetes} and \textit{Insurance} datasets, all methods perform well as surrogate models.
On datasets with more features and interactions between these, such as the \textit{Credit} dataset, we find that \ourmethod achieves a $R^2$ score of 0.9, 
while \shap, \nshap and \lime are only reach an $R^2$ score of 0.6-0.7.
On the \textit{Life Expectancy}, \textit{Student} and \textit{Breast Cancer} datasets where \shap and \lime struggle and \nshap times out, \ourmethod provides by far the best surrogate models.
This increase in performance comes with an increased computational effort for \ourmethod compared to the simpler \shap and \lime.
Still, in contrast to \nshap's limit of 16 features, \ourmethodE can explain up to 32 variables within an hour, and \ourmethodG scales up to hundreds of 
features whilst informing about the most significant interactions (right of Fig.~\ref{fig:experiments-surrogate}).

\subsection{Case Study: Covid-19 Patient}
\input{covid-case-study-table.tex}
Finally, we conduct a qualitative comparison between the Shapley value based explanations provided by \ourmethod, \shap and \nshap.
For this, we consider a Covid-19 dataset containing survival data of 1,540 hospitalized patients \citep{lambert:22:covid19}.
We train a random forest classifier to predict the likelihood of patient survival, based on diverse biomarkers such 
as age and pre-existing conditions.
For each patient we provide the respective \ourmethod, \shap and \nshap explanation in the following anonymized repository\footnote{\url{https://anonymous.4open.science/r/ICML-24-iSHAP-Case-Study-C439}}.
We show in Table \ref{tab:covid-table} the \shap (left), \nshap (middle) and \ourmethod (right) explanations for a patient which was hospitalized, and for which the model correctly predicted survival. We note that \ourmethodG and \shap explanations were computed within a second, whereas \nshap took 30 mins.

We discuss the explanations in turn. The features that \shap identifies as key to survival seem to be reasonable at first glance, but upon closer inspection are at least partly counterintuitive. \texttt{Hypertension:1} is marked as a positive factor and \texttt{Diabetes:1} as having only a slight negative effect, despite 
both are known risk factors for Covid-19 \citep{parveen:20:covid-diabetes}. A very strong positive effect is attributed to \texttt{Age:55} even though there many much younger patients in the dataset.

To obtain further insight, we move on to the \nshap explanation shown in the middle. We show the top 13 out of $2\,516$ interaction coefficients, 
the actual values of which we provide in Supplement~\ref{ap:nshap}. 
Like \shap, \nshap also identifies \texttt{Age} as a positive marginal factor, but additionally shows that it is included in many higher-order interacting sets,
interacting with \texttt{Hypertension}, \texttt{Diabetes} and \texttt{Race} amongst others in various combinations.
The amount of redundance and inconclusive values makes a clear interpretation of them hard, for example (\texttt{Age}, \texttt{Hypertension}, \texttt{Diabetes}, \texttt{Race}) is given a contribution of -9\% to survival chances,
while (\texttt{Age}, \texttt{Hypertension}) alone improve odds by 7.8\% supposedly.
The same phenomenon occurs across the board
and does not reveal which these interacting sets are, let alone which out of the $2\,516$ we should focus on.  

Thus, we inspect the \ourmethod explanation shown on the right. It partitions the feature set into six parts, and unlike the competing explanations, clearly identifies there are two interacting factors contributing to survival: the \texttt{Age:55} in conjunction with \texttt{Hypertension:1} and \texttt{Diabetes:1}, as well as the absence of \texttt{Hyperlipidemia} and \texttt{Coronary Artery Disease}. The positive interaction of age with diabetes and hypertension is explained as follows. While diabetes and hypertension are negative marginal factors for survival across the entire data set, their effect is significantly reduced for a patient of only 55 years old compared to the 
on-average much older patients in the dataset. 
This is reflected in the strong positive interaction found by \ourmethod, but also explains the high \shap value of the feature \texttt{Age:55} as it is due to the \emph{sum} of the marginal and the interaction effects. Second, we see that \texttt{Hyperlipidemia:0} and \texttt{Coronary Artery Disease:0} are positively interacting factors for survival. 
Coronary artery disease is known to be a risk factor for Covid-19 patients \citep{szarpak:22:covid-cad}, and thus its absence is a positive factor for survival. Hyperlipidemia is strongly associated with coronary artery disease \citep{goldstein:73:hyperlipidemia}, thus its absence validates the \texttt{Coronary Artery Disease:0} feature and interacts as a positive factor for survival.

Overall, we see that all three explanations describe the same phenomena, but do so in different levels of detail and succinctness. The \shap explanation is arguably the most compact, but also the least detailed as it cannot explain the interactions that are important to understand the underlying mechanisms. The \nshap explanation is arguably the most detailed, but, also the hardest to interpret, as it lists all interaction coefficients. The \ourmethod explanation offers the best of both worlds: it is as interpretable as \shap, and includes the main interactions as found by \nshap, so to succintly explain the models decision and making the user aware of the most important interactions.

%% file: figs/real_world_experiments.tex
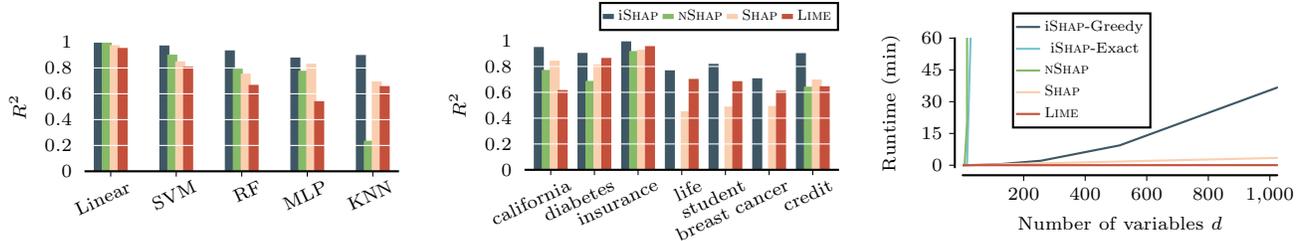
\begin{figure*}[t]
	\begin{minipage}[t]{0.33\linewidth}
    	\centering
		\vspace{-3cm}
		\begin{tikzpicture}
			\begin{groupplot}[
				pretty ybar small,
				pretty discrete tilt,
                pretty labelshift,
				cycle list name = prcl-ybar,
				ymin = 0.,
				ylabel={$R^2$},
                width=\linewidth,
				title = {},
				ymax = 1,
				group style = { group size=4 by 1, horizontal sep=15pt},
				xlabel = {},
				symbolic x coords={Linear,SVM,RF,MLP,KNN},
                legend columns = 4,
				legend style = {at={(1,1.3)},fill=white,draw=black} ]
				\nextgroupplot[width=\linewidth,height=3.3cm]
				\addplot[c1,fill=c1] table[x=Model, y=ISHAP,col sep=comma] {expres/accuracy/model_classes.csv};
                \addplot[dollarbill,fill=dollarbill] table[x=Model, y=NSHAP,col sep=comma] {expres/accuracy/model_classes.csv};
				\addplot[apricot,fill=apricot] table[x=Model, y=SHAP,col sep=comma] {expres/accuracy/model_classes.csv};
                \addplot[b5,fill=b5] table[x=Model, y=LIME,col sep=comma] {expres/accuracy/model_classes.csv};
			\end{groupplot}
		\end{tikzpicture}
		\label{fig:experiments-surrogate}
	\end{minipage}
	\begin{minipage}[t]{0.33\linewidth}
		\centering
		\begin{tikzpicture}
			\begin{groupplot}[
				pretty ybar small,
				pretty discrete tilt,
                pretty labelshift,
				cycle list name = prcl-ybar,
				ymin = 0.,
				ylabel={$R^2$},
                width=\linewidth,
				title = {},
				ymax = 1,
				group style = { group size=4 by 1, horizontal sep=15pt},
				xlabel = {},
				symbolic x coords={california,diabetes,insurance,life,student,breast cancer,credit},
				legend entries= {\ourmethod, \nshap,\shap,\lime},
                legend columns = 4,
				legend style = {at={(1,1.3)},fill=white,draw=black} ]
				\nextgroupplot[width=\linewidth,height=3.3cm]
				\addplot[c1,fill=c1] table[x=Dataset, y=ISHAP,col sep=comma] {expres/accuracy/per_dataset.csv};
                \addplot[dollarbill,fill=dollarbill] table[x=Dataset, y=NSHAP,col sep=comma] {expres/accuracy/per_dataset.csv};
				\addplot[apricot,fill=apricot] table[x=Dataset, y=SHAP,col sep=comma] {expres/accuracy/per_dataset.csv};
                \addplot[b5,fill=b5] table[x=Dataset, y=LIME,col sep=comma] {expres/accuracy/per_dataset.csv};
			\end{groupplot}
		\end{tikzpicture}
		\label{fig:experiments-datasets}
	\end{minipage}
	\begin{minipage}[t]{0.33\linewidth}
	\centering
	\vspace{-3.2cm}
	\begin{tikzpicture}
		\begin{axis}[
			pretty line,
			clip,
			width = \linewidth,
			height = 3.3cm,
			ytick = {0,15,...,60},
			ymax = 60,
			ymin = -1,
			xlabel = {Number of variables $d$},
			ylabel = {Runtime (min)},
			legend entries= {\ourmethodG, \ourmethodE, \nshap, \shap, \lime},
			legend columns = 1,
			legend style = {at={(0.6,1.2)},fill=white,draw=black,opacity=1,font= {\scriptsize},
			nodes= {scale=0.8},cells			= {scale=1}},
			thick
			]
			\addplot+[c1, thick] table[x=nvariables, y=iSHAP-Greedy, col sep=comma] {expres/full_runtimes.csv}; %
			\addplot+[c2, thick] table[x=nvariables, y=iSHAP-Exact, col sep=comma] {expres/full_runtimes.csv}; %
			\addplot+[dollarbill, thick] table[x=nvariables, y=NSHAP, col sep=comma] {expres/full_runtimes.csv}; %
			\addplot+[apricot, thick] table[x=nvariables, y=SHAP, col sep=comma] {expres/full_runtimes.csv}; %
			\addplot+[b5,thick] table[x=nvariables, y=LIME,col sep=comma] {expres/full_runtimes.csv};
			\end{axis}
		\end{tikzpicture}
	\end{minipage}
	\vspace{-0.5cm}
	\caption{$R^2$ for surrogate models of \ourmethod, \nshap, \shap and \lime across different model classes (left) and datasets (middle). \ourmethod provides
	the most accurate surrogate model across all model classes and datasets, whilst scaling to more dimension than \nshap (right).}
	\label{fig:experiments-surrogate}
\end{figure*}

%% file: covid-case-study-table.tex
\begin{table*}[!ht]
        \centering
        \subfloat[\shap values\label{}]{
            \renewcommand{\arraystretch}{1.2}
            \tiny
            \centering
            \fbox{
                \begin{tabular}{p{16ex}r}
                    Feature & Effect (\%)\\
                    \hline
                    age:55 & 28.1 \\
                    race:black & 0.9 \\
                    sex:male & -0.1 \\
                    hypertension:1 & 7.1 \\
                    hyperlipidemia:0 & -0.2 \\
                    diabetes:1 & -3.6 \\
                    CAD:0 & 7.3 \\
                    chf:0 & 0.0 \\
                    CeVD:0 & 0.0\vspace{6pt}\\
                \end{tabular}
            }
        }
        \subfloat[Top-k \nshap values (2516 total) \label{}]{
            \renewcommand{\arraystretch}{1.2}
            \tiny
            \centering
            \fbox{         
            \begin{tabular}{p{18ex}r}
                Feature Set & Effect (\%) \\
                \hline
                age & 15.1 \\
                \mbox{age, hypertension,} race, diabetes & -9.0 \\
                \mbox{age, hypertension} & 7.8 \\
                \mbox{age, race,} \mbox{hypertension} & 7.5 \\
                \mbox{age, race, diabetes} & 6.5 \\
                age, diabetes & 4.8 \\
                diabetes & -4.7 \\
                CAD & 4.7 \\
            \end{tabular}
            \begin{tabular}{p{18ex}r}
                
                Feature Set & Effect (\%) \\
                \hline
                hypertension, diabetes & 4.5 \\
                sex & -4.3 \\
                age, race, diabetes, CAD & 4.2 \\
                age,race & -4.1 \\
                age, hyperlipidemia, CAD & 3.6 \\
                $\vdots$ & $\vdots$ \vspace{1.2pt}\\
            \end{tabular}
            }
        }
        \subfloat[\ourmethod values \label{}]{
            \renewcommand{\arraystretch}{1.2}
            \tiny
            \centering
            \fbox{
            \begin{tabular}{p{16ex}r}
                Feature Set & \makecell[r]{Individual Effect+\\Interaction Effect(\%)}\\
                \hline
                \makecell[l]{diabetes:1, age:55,\\hypertension:1} & 10.2 + 18.6\\
                \makecell[l]{hyperlipidemia:0,\\CAD:0} & 4.5 + 1.9\\
                sex & -3.5 + 0\\
                race & 4.0 + 0\\
                copd & 0.2 + 0\\
                chf & 0.0 + 0\vspace{6.5pt}\\
            \end{tabular}
            }
        }
        \caption{Shapley value-based explanations for predicted Covid-19 survival odds. In (a) we show feature-wise \shap values, in (b) \nshap values for all feature subsets, and in (c) \ourmethod values, attributed to partitioned features. 
            }\label{tab:covid-table}
\end{table*}

%% file: conclusion.tex
\section{Conclusion}
In this paper, we proposed a model agnostic, post-hoc explanation method. 
In contrast to existing explanations, we directly integrate significant interactions between sets of features $x_{S_i}$ into a succinct, additive explanation.
We showed how to use a statistical test to guaranteed find the underlying optimal partitioning of features and avoid fitting spurious interactions.
Our algorithm \ourmethod is an effective and fast procedure that takes the theoretical results into practice. On synthetic data we have shown that \ourmethod 
returns accurate, ground truth interactions, and on real world data, we find that \ourmethod is in general a more 
accurate surrogate model than the state-of-the-art.

For future work, we plan to extend \ourmethod to allow multiple interactions per variable.
Furthermore, we want to extend the definition of interaction to be able to differentiate between any distribution
of interaction effects.

%% file: impact_statement.tex
\section*{Impact Statement}
\iffalse
% \iftrue
    % Standart statement provided by ICML
    This paper presents work whose goal is to advance the field of Machine Learning. There are many potential societal consequences of our work, none which we feel must be specifically highlighted here.

\else 

    In this paper we introduced \ourmethod, a model agnostic approach to explain black-box model predictions. % We have designed \ourmethod to be transparent and interpretable in execution and results, making it easy for others to understand and scrutinize the outcomes of our research.
    As our research focuses on providing understandable explanations, with the explicit goal of making the explanations interpretable our method can 
    potentially be used to discover undesired/unfair interactions that occur in a black-box model. 
    However, it is important to stress that our explanations do not make causal statements about the underlying data generating process,
    especially when the explained model is trained on sensitive data such as census.
    
\fi

%% file: appendix.tex
\section{Problem Complexity}
\label{ap:complexity}
The Subset Sum Problem is defined as follows:

\textbf{Input:} Given a set of integers $A = \{a_1, a_2, \ldots, a_n\}$ and a target integer $t$.

\textbf{Output:} Determine if there exists a subset $A' \subseteq A$ such that the sum of its elements is equal to $t$.

In our case, the set of integers $A$ is the set of all possible values of the value function $v(S)$ for $S \subseteq [d]$,
 i.e.~$A = \{v(\{1\}),v(\{1,2\}),v(\{1,2\}),\ldots, v([d])\}$.
The target integer $t$ s the prediction $f(x)-\mu$.
Note that we can transform any fixed precision floating point number into an integer by multiplying it with a sufficiently large constant.
We impose additional constraints on a valid solution, namely that the set $A'$ must be a partition of $[d]$.
We are interested in the closest subset sum variant, where instead the objective is to minimizes the difference between the sum of $A'$ and $t$,
i.e.~the prediction error $f(x)-\mu$.
Additionally, our objective contains a regularization term $\mathcal{R}(A')$ that penalizes the number of sets in the partition $A'$. 
\section{Proofs}
\label{ap:proofs}
\subsection{Proof of Theorem \ref{eq:noninteractive}}
\theoremnonadditive*
\begin{proof}
	Assume the optimal partition $\P^*$ contains a set $S$ where $i, j \in S$.
	Then, the value function $v(S)$ is decomposable into 
    \[v(S) = v(A \cup i)+v(B \cup j)\;.
    \] 
	Thus, we may construct a partition $\P'$ with $A \cup i$ and $B \cup j$.
    Let $\E[f(X)] = 0$ (achievable by pre-processing), then the reconstruction error of $\P'$ is 
    \[
        f(x) - \sum_{T' \in \P'} v(T') = \left(f(x) - \sum_{T \in \P^*} v(T) \right) - v(S) + v(A \cup i) + v(B \cup j)\;.
    \]
    As per the assumption, $v(S) = v(A \cup i)+v(B \cup j)$, so the reconstruction error of $\P'$ is equal to the reconstruction error of $\P^*$.
    For the regularization penalty of the new partition elements $A' = A \cup i$ and $B' = B \cup j$ in $\P'$ it holds that
    \[
        |A'|(|A'|-1)/2+|B'|(|B'|-1)/2 < (|A'|+|B'|)(|A'|+|B'|-1)/2\;,
    \]
    which shows that it is smaller than the regularization penalty of $\P^*$ for $S\cup i,j$.
    Thus, the overall objective of $\P'$ is lower than $\P^*$, contradicting its optimality.
	\lightning
\end{proof}
\additivityassumption*
\subsection{Proof of Theorem \ref{eq:theorem-test}}

\begin{proof}
    If there is interaction between $i$ and $j$, we show that there exists a covariate set $S$ for which $v$ is not additive for $i$ and $j$.
    First, we note that 
    \begin{gather}
        \E_S\left[\interaction(i,j,S)\right] \neq 0 \\
       \implies  \exists S \subseteq [d]\setminus \{i,j\}: \interaction(i,j,S) \neq 0 \;,
    \end{gather}
    i.e.~there exists a covariate set $S$ for which the interaction is not zero.
    For this set $S$, it holds that 
    \begin{gather}
        v(S \cup i) + v(S \cup j) \neq v(S \cup i,j) + v(S) \;. \label{eq:nonadditive-proof-apx}
    \end{gather}
    If $v$ indeed was additive for $i$ and $j$, then for $S$ there exists a partition $A \cup B = S$ so that
    \[
        v(S \cup i, j) = v(A \cup i) + v(B \cup j)\;.	
    \]
    By Assumption \ref{as:additivity-assumption}, we know that this decomposition also holds for $S$, $S \cup i$ and $S \cup j$, so that we can rewrite Equation \eqref{eq:nonadditive-proof-apx} as
    \begin{gather}
        v(A \cup i) + v(B) + v(A) + v(B \cup j) \\
        \neq v(A \cup i) + v(B \cup j) + v(A) + v(B)\;.
    \end{gather}
    This statement is a contradiction, and thus proves that $v$ is not additive for $i$ and $j$.
    \end{proof}
\subsection{Proof of Theorem \ref{eq:theorem-multiple}}
\theoremmultiple*
\begin{proof}
	Assume that $v$ is additive for $i$ and $k$, i.e.
	$\forall S_3 \subseteq [d]\setminus\{i,k\}: \exists A, B: v(A \cup i) + v(B \cup k) = v(S_3 \cup i, k).$
	
	Now consider a set $S_1$ for which $v$ is not additive in regards to $i$ and $j$, i.e.~$\forall A_1, B_1: v(A_1\cup i)+v(B_1 \cup j) \neq v(S_1 \cup i,j)$ and the set $S_2$ for $j$ and $k$ respectively.
	We construct a set $S_3 = (S_1 \cup S_2 \cup \{j\})\setminus\{i, k\}$, for which $v$ now has to be additive with regard to $i$ and $k$, i.e.~there exists a partition $A_3$, $B_3$ so that 
	$v(A_3 \cup i) + v(B_3 \cup k) = v(S_3 \cup i, k)$.
	
	There are two cases to consider, either $j \in A_3$ or $j \in B_3$.
	Let $j \in B_3$, then we can construct a new sub-partition $A_1 = S_1 \cap A_3$ and $B_1 = S_1 \cap (B_3\cup k)$.
	$A_1$ and $B_1$ are subsets of $A_3\cup i$ and $B_3\cup k$, so that by Assumption 1 additivity is preserved for $A_1$ and $B_1$. Therefore, it holds that 
	$v(A_1 \cup i) + v((B_1 \setminus j) \cup j) = v(S_1 \cup i, j)$, since $A_1\cup B_1 = S_1$ as $S_1 \subseteq S_3 \cup k$.
	This contradicts the assumption that $v$ is non-additive for $i$ and $j$.
	
	Similarly, we can show that $j \in A_3$ violates the assumption that $v$ is non-additive for $j$ and $k$, and conclude that $v$ must in fact be non-additive for $i$ and $k$.
	In the interaction graph, this allows us to reject the additivity of a pair of variables $i$ and $j$ if they are connected by a path, 
	and justifies the use of connected components over cliques as search space.
\end{proof}

\section{Additivity of Value Functions}
\label{ap:additivity}
We consider two value functions: the observational value function
$$v(S;f,x) = \E\left[f(X)|X_S=x_S\right]$$ by \cite{lundberg:17:shap}, where it is assumed that all variables are independent of each other, i.e.~$\forall i \neq j: X_i \ind X_j$,
and the interventional value function $$v(S;f,x) = \E\left[f(X')|\Do(X_S'=x_S)\right]$$ by \cite{janzing:20:causalshapley},
where we consider as features variables the model inputs $X_i'$ that are purely determined by the real world counterpart $X_i$.
Hence, intervening on the model input as $\Do(X_S'=x_S)$ only has an effect on the input $X_S'$.

Let $f$ now be additive for two sets $A$ and $B$, so that $f(x)= g(x_A)+h(x_B)$, then the value function $v(A;f,x) + v(B;f,x)$ is transformed into
\begin{align}
    & \E\left[f(X)|X_A=x_A\right] + \E\left[f(X)|X_B=x_B\right] \\
    = & \E\left[g(X)| X_A=x_A\right] + \E\left[h(X)| X_A=x_A\right] + \E\left[g(X)|X_B=x_B\right] + \E\left[h(X)|X_B=x_B\right]\;.
\end{align}
We can drop the conditioning of $X_A=x_A$ where there is only $h$ and vice versa for $X_B=x_B$ and $g$.
This is possible both with the independence assumption in the observational Shapley values by \cite{lundberg:17:shap}, and the causal model 
as postulated by \cite{janzing:20:causalshapley}. This leaves us with
\begin{align}
    \E\left[g(X)| X_A=x_A\right] + \E\left[h(X)\right] + \E\left[g(X)\right] + \E\left[h(X)|X_B=x_B\right] = \E\left[g(X)| X_A=x_A\right] + \E\left[h(X)|X_B=x_B\right] + \mu\;.
\end{align}
By convention, we preprocess $f(X)$ so that $\mu = \E[f(X)] = 0$, and $\mu$ hence can be dropped.
Now, we similarly decompose $v(S;f,x)$ as 
\[
    \E[f(X)|X_S=x_S] = \E[g(X)+h(X)|X_S=x_S] = \E[g(X)|X_S=x_S] + \E[h(X)|X_S=x_S] \;.
\]
Now we again drop $X_B$ from $g$ and $X_A$ from $h$ and are left with
\[
    \E[g(X)|X_A=x_A] + \E[h(X)|X_B=x_B]\;,
\]
which shows the additivity for any set $A$ and $B$ for which $f$ is decomposable.
\section{Algorithm}
\label{ap:algorithm}
\ourmethod consists of two main subroutines: find\_interactions and find\_partition.
The first subroutine is the same for both the greedy and the optimal algorithm.
\ourmethodG uses a greedy, bottom-up approach to find the best partition, 
while \ourmethodE uses a exhaustive search.

\paragraph{find\_interactions.}
As input, find\_interactions receives a data point $x$, a model $f$ and a sample $\hat{X}$ of $P(X)$.
It returns all pairwise interactions between features that are statistically significant, encoded 
as a graph $G$.
We initialize the interaction graph $G$ with $d$ nodes, where each node represents a single feature.
We then sample $n_s$ new data points $x^{(j)'}$ from the empiric data distribution, either marginally, conditionally or interventional as 
required by the value function $v$.
For each data point $x^{(j)'}$, we sample a random intervention $z \in \{0,1\}^d$ with $p=0.5$.
We then intervene on the $i$-th feature on the $j$-th data point $x^{(j)'}$, i.e.~$x^{(j)'}_i=x_i$, if $z^{(j)}_i=1$.

Now, for each pair of features $i,j$, we test the hypothesis
\[
    H_0: \E_S\left[v(S \cup i)-v(S)+v(S \cup j)-v(S)\right] = \E_S\left[v(S \cup i,j)-v(S)\right]
\]
by taking dividing up the sample $\{f(x^{(j)'})\}$ into four subsets according to the intervention $z^{(j)}$:
\begin{itemize}
    \item $v(S \cup \{i,j\}) \gets \{f(x^{(j)'}) | i,j \in z^{(j)} \}$
    \item $v(S \cup \{i\})\gets\{f(x^{(j)'}) | i \in z^{(j)}, j \notin z^{(j)} \}$
    \item $v(S \cup \{j\})\gets\{f(x^{(j)'}) | j \in z^{(j)}, i \notin z^{(j)} \}$
    \item $v(S)\gets\{f(x^{(j)'}) | i,j \notin z^{(j)} \}$
\end{itemize}
Now, we can test the hypothesis using a two-sided t-test with significance level $\alpha$ with unequal variances,
also known as Welch's t-test.
If the hypothesis is rejected, we add an edge between the nodes $i$ and $j$ to the graph $G$.
By splitting up the sample into four subsets, we can test the hypothesis for each pair of features $i,j$ 
on the same sample, which is more efficient than testing each pair on a separate sample as 
it reduces the number of required samples and thus evaluations of $f$ by the amount of pairwise interactions, i.e.~$O(d^2)$.

\subsection{Complexity}
The complexity of the greedy search is cubic. For a fully connected graph we have to evaluate $d(d-1)/2$ merges, where we can take at most $d$ steps
before arriving at the complete set $[d]$. In each step, we have to estimate the value function $v$ with $n$ samples, 
whose complexity we denote by $O(v(n,d))$. Hence, its complexity is $O(d^3)O(v(n,d))$. 

\begin{algorithm2e}[t]
	\caption{find\_interactions($f,x,\hat{X},n_S,\alpha$)}
    \label{alg:interactions}
    \KwIn{Data point $x$, model $f$, sample $\hat{X}$, number of samples $n_S$, significance level $\alpha$}
    Initialize graph $G$ with $d$ nodes.\\
    Sample $n_s$ interventions $z^{(k)} \in \{0,1\}^d$ with $p=0.5$.\\
    Sample $n_s$ new points $x^{(k)'}$ from the marginal/conditional/interventional distribution using $\hat{X}$.\\
    \For{$i=1:n_s$, $j=1:d$} {
        \If{$z^{(i)}_j=1$}{
            Set the $i$-th feature on the $j$-th data point $x^{(j)'}_i$ to $x_i$
        }
    }
    \For{$i=1:d$, $j=i+1:d$} {
        Sample of $v(S \cup \{i,j\}) \gets \{f(x^{(k)'}) | i,j \in z^{(k)} \}$\\
        Sample of $v(S \cup \{i\})\gets\{f(x^{(k)'}) | i \in z^{(k)}, j \notin z^{(k)} \}$\\
        Sample of $v(S \cup \{j\})\gets\{f(x^{(k)'}) | i \notin z^{(k)}, j \in z^{(k)} \}$\\
        Sample of $v(S)\gets\{f(x^{(k)'}) | i,j \notin z^{(k)} \}$\\
        p = Welch's t-test for $v(S \cup \{i\}) + v(S \cup \{j\}) = v(S \cup \{i,j\}) + v(S)$\\
        \If{$p < \alpha$}{
            Add edge $(i,j)$ to $G$
        }
    }
    \Return $G$
\end{algorithm2e}

\paragraph{find\_partition.}
The find\_partition subroutine takes in addition the graph $G$ from find\_interactions
and uses the same data point $x$, the model $f$ and the sample $\hat{X}$ of $P(X)$. 
find\_partition returns the best scored partition $\P$ in regards to Objective \ref{eq:objective}.

For the greedy approach, we initialize $\P$ with $d$ singleton sets, i.e.~$\P = \{S_i|S_i=\{i\}\}_{i=1}^d$.
We merge all eligible pairs of sets $S_i,S_j \in \P$ into a new set $S_i \cup S_j$ and score the new partition $\P'$.
Eligibility is given if the graph $G$ contains an edge between an element of $S_i$ and an element of $S_j$.
Each step, we merge the pair of sets that yields the best score and terminate 
once no more improvement is possible.

The exhaustive approach is a brute-force search over all possible partitions $\P$.
We restrict the search space by only considering partitions $\P$, were all elements $S_i \in \P$ are connected in the graph $G$.
Then, we score each partition $\P$ and return the best scored partition.

\begin{algorithm2e}[]
    \caption{score\_partition($\P,v,\lambda$,f(x))}
    \label{alg:score_partition}
    Pred $\gets 0$\\
    \For{$S_i \in \P$} {
        Pred $\gets$ Pred + $v(S_i)$
    }
    Error $\gets$ Pred $- f(x)$\\
    Reg $\gets \lambda \cdot \sum_{S_i \in \P} |S_i|$\\
    \Return Error + Reg
\end{algorithm2e}

\begin{algorithm2e}
	\caption{find\_partition\_greedy($G,v,f(x),\lambda$)}
    \label{alg:partition-greedy}
    \KwIn{Graph $G$, value function $v$, prediction $f(x)$, regularization parameter $\lambda$}
    d $\gets$ number of nodes in $G$\\
    Initialize $\P = \{S_i|S_i=\{i\}\}_{i=1}^d$.\\
    \While{True}{
        Best Score $\gets$ score\_partition($\P,v,\lambda$,f(x))\\
        Candidate Found $\gets$ False\\
        \For{$S_i,S_j \in \P$} {
            \If{$\exists k \in S_i, l \in S_j: (k,l) \in G$}{
                Candidate $\P' \gets \P$\\
                $\P' \gets \P' \cup \{S_i \cup S_j\} \setminus \{S_i,S_j\}$\\
                Candidate Score = score\_partition($\P',v,\lambda$,f(x))\\
                \If{Candidate Score $<$ Best Score}{
                    Best Score $\gets$ Candidate Score\\
                    $\P \gets \P'$\\
                    Candidate Found $\gets$ True\\
                }
            }
        }
        \If{Candidate Found = False}{
            Break
        }
    }
    \Return $\P$
\end{algorithm2e}

\begin{algorithm2e}[]
	\caption{find\_partition\_exhaustive($G,v,f(x),\lambda$)}
    \label{alg:partition-exhaustive}
    \KwIn{Graph $G$, value function $v$, prediction $f(x)$, regularization parameter $\lambda$}
    d $\gets$ number of nodes in $G$\\
    Best Score $\gets \infty$\\
    \For{$\P$ in all partitions of $[d]$}{
        \If{$\exists S_i,S_j \in \P: \nexists k \in S_i, l \in S_j: (k,l) \in G$}{
           \Continue
        }
        Score $\gets$ score\_partition($\P,v,\lambda,f(x)$)\\
        \If{Score $<$ Best Score}{
            Best Score $\gets$ Score\\
            Best $\P \gets \P$
        }
    }
    \Return Best $\P$
\end{algorithm2e}

\begin{algorithm2e}
	\caption{$\ourmethod(f,x,\hat{X},n_S,v,\alpha,$greedy)}
    \KwIn{Data point $x$, model $f$, sample $\hat{X}$, number of samples $n_S$, value function $v$, significance level $\alpha$}
	\label{alg:ishap}
        Interaction Graph $G \gets$ find\_interactions($f,x,\hat{X},n_S,\alpha$)\\
        \If{greedy}{
            Partition $\P \gets$ find\_partition\_greedy($G,v,f(x),\lambda$)\\
        }
        \Else{
            Partition $\P \gets$ find\_partition\_exhaustive($G,v,f(x),\lambda$)\\
        }
        Value function $v'(S)$ of game with $\P$ as players, $v'(S) = v(S)$ if $S \in \mathcal{P}(\P)$\\
        Compute Shapley values $\phi_i$ for $i=1,\dots,m$ using $S_i \in \P$ as players using $v'$\\
        \Return Explanation $\{(S_i,\phi_i)\}_{S_i \in \P}$, Interaction Graph G
        
\end{algorithm2e}

\paragraph{Explanation.}
Once we have found the best scored partition $\P$, we can explain the prediction $f(x)$ by
computing the Shapley values for a new game $v'$ which has as its players the elements of $\P$ instead
of the features $\{1,\ldots,d\}$.
The value function $v'$ is defined only for those sets $S$ which are elements of the power set of $\P$.
On that set, the value function $v'$ is defined as $v'(S) = v(S)$ as the underlying model $f$ is the same.

On this new game, we can compute the Shapley values $\phi_i$ for each element $S_i \in \P$.
The Shapley value $\phi_i$ is the average marginal contribution of the element $S_i$ to the value of the game $v'$.
\ourmethod returns the Shapley values $\phi_i$ as the explanation for the prediction $f(x)$, 
in addition to the interaction graph $G$ and the partition $\P$.

\section{Experiments}
\label{ap:experiments}
\subsection{Interaction Experiments}
We first verify whether \ourmethod accurately recovers the correct sets of variables for a generalized additive model. 
To this end, we generate a random function $f$ over $d \in \{4,6,8,10,15,20,30,50,75,100\}$ variables. 
To obtain a function with additive components, we partition the variables into sets $S_i$, where we iteratively sample the size of each set from 
a Poisson distribution with $\lambda = 1.5$, 
over which we define function $f(X) = \sum_{S_i \in \P} f_i(X_{S_i})$, whereas inner functions $f_i$ we consider
\begin{itemize}
    \setlength\itemsep{0em}
     \item $f_i\left(X_{S_i}\right)= \prod_{j\in S_i} a_{i,j}\cdot X_j$, $a_{i,j}\in \pm[0.5,1.5]$
     \item $f_i\left(X_{S_i}\right)= \sin\left(\sum_{j\in S_i} a_{i,j}\cdot X_j\right)$, $a_{i,j}\in \pm[0.5,1.5]$
 \end{itemize}
We sample all the underlying $d$ variables either from a normal distribution: $P(X_j)=N(\mu,\sigma^2), \mu \in [0,3], \sigma \in [0.5,1.5]$ or a uniform distribution: $X_j \in \mathcal{U}(0,3)$. 
In total, for we test the accuracy of 100 explanations for each combination of $d$, inner function and sampling distribution.
Each time, we sample a random function $f$ and dataset $X$ of 10$\,$000 points. 
We use the observational value function $v(S) = E[f(X)|X_S=x_S]$,
where we sample $X_i$ individually as they are independently generated.
For a random $x \in X$, we generate the partition $\hat{P}$ with \ourmethod, 
using as significance level $\alpha=0.01$ and as regularization coefficient
$\lambda = 5\mathrm{e}{-3}$. 

\paragraph{Ground Truth Feature Importance}
\input{figs/ground_truth.tex}
Next, we evaluate whether the reported \ourmethod values align with the ground truth feature importance, to this end, we first need to define the ground truth feature importance for sets of variables. For purely linear functions the importance of a single feature is $v(i, x', f)=w_ix'_i - w_i\E[X_i]$, where $w_i$ is the weight of feature $i$ in $f$. We use the same concept for general additive models, where we use $f_i\left(X_{S_i}\right)= \prod_{j\in S_i} a_{i,j}\cdot X_j$, $a_{i,j}\in \pm[0.5,1.5]$ as inner functions, hence the importance of a feature set $S_i$ for a given sample is, 
\[
    v(S_i,x',f) = \frac{\partial^{|S|} f}{\prod_{j \in S_i}{\partial x_j}} (x'_1, \dots, x'_d) * \prod_{j \in S_i}x'_j \quad .
\]
Since the expected value is constant offset, we omit it. 
We compute person correlation between the \ourmethod explanation and the ground truth $v(S_i,x',f)$.

We consider the same setup as described above, with $d=20$, we generate 100 random models, for each model we uniform at random select 300 datapoints and generate the respective explanations. In Figure \ref{fig:feature-importance} (left) we show the average of the average correlation over all models. For comparison, we conduct the analog experiment, over the same set of models and datapoints, with \shap values, computing the correlation between \shap values and the respective partial derivative times the feature value.
We observe a near-perfect correlation between ground truth and \ourmethod values, while the \shap values are less correlated with their respective gradients. 

In Figure \ref{fig:feature-importance} (middle) we show the correlation between the individual additive components, for a specific model, and on the right the correlation of \shap values with the respective gradient times feature values. 
While \ourmethod gets near perfect correlation, \shap gets a much lower correlation on average and only gets near near-perfect correlation for additive components that include only a single variable. 

With this experiment, we show that \ourmethod gives better insight into general additive models and models that can (locally) be approximated with GAM's. 

\subsection{Accuracy Experiments}
We consider five regression and two classification dataset:\textit{California}~\citep{pace:97:california}, \textit{Diabetes}~\citep{pedregosa:11:sklearn}, \textit{Insurance}~\citep{lantz:19:ml-r}, \textit{Life}~\citep{rajarshi:19:who}, \textit{Student}~\citep{cortez:08:student}
with increasingly complex, non-additive models $f$: linear models, support vector machines (SVM), random forests (RF), multi-layer perceptrons (MLP) and k-nearest neighbors (KNN).
This experiment is for each dataset repeated 100 times and goes as follows: we pick two instances $x^{(1)}$ and $x^{(2)}$ from a dataset. We use these to construct a new data point $x'$ by uniformly at random choosing the value for each variable $X_i$ either as defined by $x^{(1)}$ or by $x^{(2)}$. Let $I_1$ be the indices of all feature values chosen from $x^{(1)}$, and $I_2$ from $x^{(2)}$. We run the respective explanation method for both $x^{(1)}$ and $x^{(2)}$. We can now use these additive explanations to obtain the prediction of the surrogate model for $x'$ defined as 
By design, this scheme is not applicable for all random instances with partitions $\P^{(k)}$ and $\P^{(l)}$, in which case simply resample until we have an admissible input.
We compute the implied prediction for the new data point $x'$ in the following way for each method:
\begin{itemize}
    \setlength\itemsep{0em}
    \item \ourmethod: $f(x') = \sum_{S_i \in \P_1, S_i \subset I_1} e_{i}^{(1)} + \sum_{S_j \in \P_2, S_j \subseteq I_2} e_{j}^{(2)}$
    \item \shap: $f(x') = \sum_{i \in I_1} \phi_i^{(1)} + \sum_{j \in I_2} \phi_j^{(2)}$
    \item \nshap: $f(x') = \sum_{S \subseteq I_1} \Phi_{S}^{(1)} + \sum_{S \subseteq I_2} \Phi_{S}^{(2)}$
    \item \lime: For each datapoint we obtain a surrogate model $\hat{f}_1$ and $\hat{f}_2$. We take $\hat{f}_1(x')$ and $\hat{f}_2(x')$ and 
    use whichever is closer to $f(x')$.
\end{itemize}

\section{Ablation Study}
\label{ap:ablation}

We seek to determine how \ourmethod between the interaction test, or whether a simple, greedy merging algorithm alone already suffices.
To this end, we perform an ablation study of \ourmethod with and without interaction testing on the same synthetic data, both for the \ourmethodG and \ourmethodE search.
We use the same parameters as in the previous section.
When ablation the interaction test, we simply skip the interaction test and use a fully connected interaction graph $G$ instead.

We observe an average reduction of \textbf{53\%} in runtime for the greedy search and by  \textbf{95\%} for the full search.
This can be explained the reduction in search space, because over 99\% less value functions need to be computed
as well as 99\% less partitions for the exhaustive search.
This gain is partially offset by the interaction test, which introduces an overhead that grows quadratically with the number of variables,
but is asymptotically offset by the exponential amount of value functions/ super-exponential amount of partitions.
Finally, we compare the average $\mathit{F1}$ score with interaction testing against without.
The interaction test improves the $\mathit{F1}$ score by \textbf{47\%} for the greedy search and \textbf{20\%} for the full search.
This shows that the interaction test is crucial to obtain accurate explanations, especially for the greedy search.

\section{\nshap Values Table for Bike Sharing and Covid-19 Explanations}
\label{ap:nshap}
We show an excerpt of the \nshap values for the bike sharing in Table \ref{tab:bike}
 and Covid-19 examples in Table \ref{tab:covid}.
The full tables are available in csv format in the Supplementary Material.

\begin{figure}
\begin{minipage}[t]{0.4\linewidth}
    \centering
\begin{tabular}{cc}
        \hline Features & \nshap-Value \\
\end{tabular}
\end{minipage}
\hfill
\begin{minipage}[t]{0.4\linewidth}
    \centering
    \begin{tabular}{cc}
    \hline Features & \nshap-Value \\ \hline
    \csvreader[
    separator=semicolon,
    late after line=\\ \hline,
    filter expr={test{\ifnumgreater{\thecsvinputline}{55}} and test{\ifnumless{\thecsvinputline}{109}}}
    ]{expres/bike_qualitative.csv}{}{\csvcoli & \csvcolii}
    \dots
    \end{tabular}
\end{minipage}
    \caption{\nshap values for bike sharing prediction.}
    \label{tab:bike}
\end{figure}

\begin{figure}
    \begin{minipage}[t]{0.38\linewidth}
        \centering
        \begin{tabular}{cc}
            \hline Features & \nshap-Value \\
\end{tabular}
    \end{minipage}
    \hfill
    \begin{minipage}[t]{0.6\linewidth}
        \centering
        \begin{tabular}{cc}
        \hline Features & \nshap-Value \\ \hline
        \csvreader[
        separator=semicolon,
        late after line=\\ \hline,
        filter expr={test{\ifnumgreater{\thecsvinputline}{54}} and test{\ifnumless{\thecsvinputline}{107}}}
        ]{expres/covid_nshap_complete.csv}{}{\csvcoli & \csvcolii}
        \dots
        \end{tabular}
    \end{minipage}
    \caption{\nshap values for covid-19 survival prediction.}
        \label{tab:covid}
    \end{figure}

%% file: figs/ground_truth.tex
\begin{figure}
    \begin{tikzpicture}
        \def\n{2} %
       \begin{axis}[
           my boxplot,
           ymax = 1,
           ymin = 0,
           xmin = 0, 
           xmax = 3,
           xtick={0,1,2,3,4,5,6,7,8,9,10},
           xticklabels={,\ourmethod, \shap},
           width = 5cm,
           height = 5cm,
           ylabel = {Average Correlation},
       ]
       \addplot+[my box]
       table [
               y=ishap,
               col sep=comma,
           ]
           {expres/avrage_correlation_grad.csv};
           \addplot+[my box]
           table [
                   y=shap,
                   col sep=comma,
               ]
               {expres/avrage_correlation_grad.csv};
       \end{axis}
    \end{tikzpicture}
    \begin{tikzpicture}
        \begin{axis}[
            ymin = 0,
            ymax = 1.1,
            width = 5cm,
            height = 5cm,
            pretty ybar small,
            xlabel near ticks,
            x label style 		= { font=\scriptsize },
            x tick label style 	= { font=\scriptsize,
                rotate=50, anchor=east, align=right },
            bar width       = 0.2cm,
            cycle list name = prcl-ybar,
            ylabel          = {Pearson Correlation}, xlabel = {}, 
            legend entries  = {},
            xticklabel style={
                xshift=-3pt, %
            },
            xticklabels = {$\{x_{13}\}$,$\{x_{18}\}$,$\{x_{1}\}$,$\{x_{14}\, x_{6} \, x_{15} \,x_{17}\}$,$\{x_{7} \, x_{12} \,x_{9}\,x_{2}\}$,$\{x_{19} \,x_{0}\}$,$\{x_{3}\}$,$\{x_{10} \,x_{5} \, x_{8}\}$,$\{x_{4}\}$,$\{x_{11} \, x_{16}\}$},
            ytick={0,0.2,0.4,0.6,0.8,1},
            visualization depends on=y \as \y,
             ]
            \addplot table[ y=correlation, col sep=comma,] {expres/var_set_correlation_example.csv};
        \end{axis}
    \end{tikzpicture}
    \begin{tikzpicture}
        \begin{axis}[
            ymin = 0,
            ymax = 1.1,
            height = 5cm,
            width = 8cm,
            pretty ybar small,
            xlabel near ticks,
            x label style 		= { font=\scriptsize },
            x tick label style 	= { font=\scriptsize,
                rotate=50, anchor=east, align=right },
            bar width       = 0.2cm,
            cycle list name = prcl-ybar,
            ylabel          = {Pearson Correlation}, xlabel = {}, 
            legend entries  = {},
            xticklabels  = {$x_{0}$,$x_{1}$,$x_{2}$,$x_{3}$,$x_{4}$,$x_{5}$,$x_{6}$,$x_{7}$,$x_{8}$,$x_{9}$,$x_{10}$,$x_{11}$,$x_{12}$,$x_{13}$,$x_{14}$,$x_{15}$,$x_{16}$,$x_{17}$,$x_{18}$,$x_{19}$},
            ytick={0,0.2,0.4,0.6,0.8,1},
            visualization depends on=y \as \y,
            ]
            \addplot table[ y=correlation, col sep=comma,] {expres/var_correlation_example.csv};
        \end{axis}
    \end{tikzpicture}
    \caption{Comparison between ground truth feature importance (gradient times feature values) and the respective \ourmethod values and \shap values. On the left, we show the average correlation over all models. In the middle, we show the correlation between the \ourmethod values and the ground truth for a specific model, on the right, we show the correlation between the ground truth and the \shap values.}
    \label{fig:feature-importance}
\end{figure}
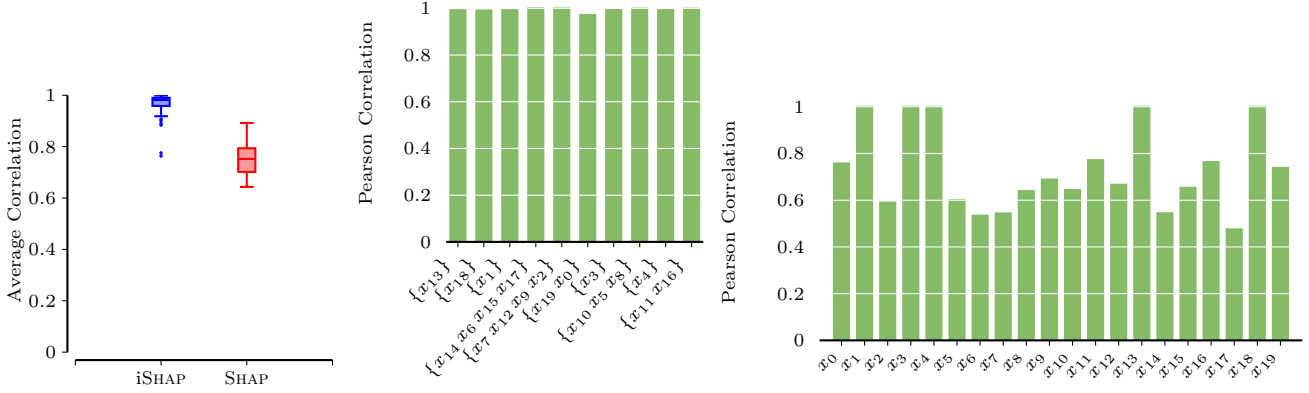